\documentclass[10pt]{article}
\usepackage{graphicx}
\usepackage{amsmath, amssymb, amsthm}

\newcommand{\tx}{\texttt{X}}

\newcommand{\ttt}{\texttt{T}}

\newcommand{\il}{\mathcal{IL}}
\newcommand{\studeny}{Studen{\'y}}
\newtheorem{thm}{Theorem}[section]
\newtheorem{lemma}{Lemma}[section]
\newtheorem{prop}{Proposition}[section]
\newtheorem{coro}{Corollary}[section]

\theoremstyle{definition}

\newtheorem{dfn}{Definition}[section]
\newtheorem{ex}{Example}[section]
\theoremstyle{remark}

\begin{document}

\title{Causal models have no complete axiomatic
characterization
}
\author{Sanjiang Li \\
Department of Computer Science and Technology,\\
Tsinghua University, Beijing 100084, China\\
lisanjiang@tsinghua.edu.cn}

\date{}
\maketitle

\begin{abstract}

Markov networks and Bayesian networks are effective graphic
representations of the dependencies embedded in probabilistic
models. It is well known that independencies captured by Markov
networks (called graph-isomorphs) have a finite axiomatic
characterization. This paper, however, shows that independencies
captured by Bayesian networks (called causal models) have no
axiomatization by using even countably many Horn or disjunctive
clauses. This is because a sub-independency model of a causal model
may be not causal, while graph-isomorphs are closed under
sub-models.

\vspace*{2mm}


\end{abstract}


\section{Introduction}
The notion of conditional independence (CI) plays a fundamental role
in probabilistic reasoning. In traditional theories of probability,
to decide if a CI statement holds, we need to check whether two
conditional probabilities are equal, which require summations over
exponentially large number of variable combinations. This numerical
approach is clearly impractical. An alternative qualitative approach
is very popular in artificial intelligence, where new CI statements
can be derived logically without reference to numerical quantities.
Given an initial set of independence relations, a fixed (finite) set
of axioms can be used to infer new independencies by logical
manipulations.

A natural question arises: can CI relations be completely
characterized by a finite set of axioms (or called inference rules)?
Pearl and Paz \cite{PearlP85} introduced the concept of
semi-graphoid as an independency model that satisfies four specific
axioms, and showed that each CI relation is a semi-graphoid. Later,
\studeny\ \cite{Studeny92} gave a negative answer to this question.
But, more positively, he also showed that (i) CI relations have a
characterization by a countable set of axioms \cite{Studeny92}; and
(ii) every probabilistically sound axiom with at most two
antecedents is a consequence of the semi-graphoid axioms
\cite{Studeny97}.

Although CI relations in general have no complete axiomatic
characterization, Geiger and Pearl \cite{GeigerP93} developed
complete axiomatizations for saturated independence and marginal
independence -- two special families of CI relations.

Graphs are the most common metaphors for communicating and reasoning
about dependencies. It is not surprising that graphical models is a
very popular way of specifying independence constraints. There are
in general two kinds of graphical models: Markov networks and
Bayesian networks. A Markov network is an undirected graph, while a
Bayesian network is a directed acyclic graph (DAG). Geiger and Pearl
\cite{GeigerP93} developed an axiomatic basis for the relationships
between CI and graphic models in statistic analysis. They showed in
particular that (i) every axiom for conditional independence is also
an axiom for graph separation; and (ii) every graph represents a
consistent set of independence and dependence constraints. Moreover,
an early work of Pearl and Paz \cite{PearlP85} gave an axiomatic
characterization for CI relations captured by undirected graphs
(called \emph{graph-isomorphs}). It was also conjectured
\cite{Pearl88} that CI relations captured by DAGs (called
\emph{causal models}) may have no finite axiomatic characterization.

In this paper, we confirm this conjecture and show that causal
models have no complete characterization by any (finite or
countable) set of (Horn or disjunctive) axioms. We achieve this by
showing that a sub-model of a causal model can be not causal. This
is contrasted by CI relations and graph-isomorphs. Both are closed
under sub-models.

It came to us very late that the same observation has been made in
\cite[Remark 3.5]{Studeny05}, where \studeny\ gave just basic
argument. This paper will provide a complete proof for this
observation.

The remainder part of this paper proceeds as follows. Section~2
provides preliminary definitions for independency models, CI
relations, graph-isomorphs, and causal models. Section~3 gives
syntactic and semantic descriptions of independency logic, and then
formalizes the notion of axiomatization. Then in Section~4 we
discuss heredity property of independency models. Further
discussions are given in the last section.
\section{Preliminaries}
In this section we introduce the basic notions used in this paper.
Our reference is \cite{Pearl88,Pearl00}. In what follows, if not
otherwise stated we assume $U$ is a finite set, and write $\wp(U)$
for the powerset of $U$.

The notion of conditional independency (CI) plays a fundamental role
in probabilistic reasoning.

\begin{dfn}[conditional independency, CI]\label{dfn:prob_model}
Let $U$ be a finite set of variables with discrete values. Let
$P(\cdot)$ be a joint probability function over the variables in
$U$. For three disjoint subsets $X,Y,Z$ of $U$, $X$ and $Y$ are said
to be \emph{conditional independent given $Z$} if for all values
$x,y$ and $z$ such that $P(y,z)>0$ we have $P(x|y,z)=P(x|z)$.
\end{dfn}
We use the notation $I(X,Z,Y)_P$ to denote the conditional
independency of $X$ and $Y$ given $Z$. The set of all these CI
statements form a ternary relation on $\wp(U)$, called a CI
relation. In general, we have
\begin{dfn}[independency model \cite{Pearl88}]\label{dfn:independency_relation}
An independency model $M$ defined on $U$ is a ternary relation on
$\wp(U)$ which satisfies the following condition:
\begin{equation}\label{eq:disjoint}
(A,C,B)\in M\ \Rightarrow\ A, B, C\ \mbox{are\ pairwise\ disjoint.}
\end{equation}
A tuple $(A,C,B)$ in $M$ (out of $M$, resp.) is called an
independence statement (a dependence statement, resp.). We write
$I(A,C,B)_M$ to indicate the fact that $(A,C,B)$ is in $M$.
\end{dfn}

Two other classes of independency models arise from graphs, where
the notion of separation plays a key role.

\begin{dfn}[graph separation \cite{Pearl88}]
If $A,B$ and $C$ are three disjoint subsets of nodes in an
undirected graph $G$, then $C$ is said to separate $A$ from $B$,
denoted $\langle A|C|B\rangle_G$, if along every path between a node
in $A$ and a node in $B$ there is a node in $C$.
\end{dfn}
The independency model consisting of all graph separation instances
in $G$ is a graph-isomorph.
\begin{dfn}[graph-isomorph
\cite{Pearl88}]\label{dfn:graph-isomorph} An independency model $M$
is said to be a graph-isomorph if there exists an undirected graph
$G=(U,E)$ such that for every three disjoint subsets $A,B,C$ of $U$,
we have
\begin{equation}\label{eq:graph-separation}
I(A,C,B)_M\Leftrightarrow\langle A|C|B\rangle_G.
\end{equation}
\end{dfn}
For directed acyclic graphs, a similar separation property was
defined.
\begin{dfn}[$d$-separation \cite{Pearl88}]
If $A,B$ and $C$ are three disjoint subsets of nodes in a DAG $D$,
then $C$ is said to $d$-separate $A$ from $B$, denoted $\langle
A|C|B\rangle_D$, if along every path between a node in $A$ and a
node in $B$ there is a node $w$ satisfying one of the following two
conditions:
\begin{itemize}
\item  $w$ has converging arrows and none of $w$ or its
descendants are in $C$; or

\item $w$ does not have converging arrows and $w$ is in $C$.
\end{itemize}
\end{dfn}
The independency model consisting of all $d$-separation instances in
a DAG $D$ is a causal model.
\begin{dfn}[causal model \cite{Pearl88}]\label{dfn:causal}
An independency model $M$ is said to be \emph{causal} if there is a
DAG $D$ such that for every three disjoint subsets $A,B,C$ of $U$,
we have
\begin{equation}\label{eq:graph-separation}
I(A,C,B)_M\Leftrightarrow\langle A|C|B\rangle_D.
\end{equation}
\end{dfn}
It was proved by Geiger and Pearl that, for every graph-isomorph
(causal model) $M$ on $U$, there is a probability distribution $P$
on $U$ such that $M$ is precisely the CI relation induced by $P$
\cite{GeigerP88,GeigerP93}.

\section{Independency logic}
To formalize the notion of axiomatization, we introduce the
independency logic $\il$. Although $\il$ is a fragment of
first-order logic, we are mainly concerned with its propositional
counterpart.

The language of $\il$ has as its alphabet of symbols:
\begin{itemize}
\item variables
$\tx_1,\tx_2,\cdots$;

\item the constant $\varnothing$;

\item the ternary predicate $I$;

\item three function letters: $-,\cup,\cap$;

\item the punctuation symbols (,) and ,;

\item the connectives $\neg,\vee,\wedge$
\end{itemize}
Terms in the independency logic are defined as follows.
\begin{dfn}[term]\label{dfn:term}
A term in $\il$ is defined as follows.
\begin{itemize}
\item [(i)] Constant and Variables are terms.
\item [(ii)] If $\ttt_1,\ttt_2$ are terms in $\il$, then
$-\ttt_1,\ttt_1\cup\ttt_2,\ttt_1\cap\ttt_2$ are terms in $\il$.
\item [(iii)] The set of all terms is generated as in (i) and
(ii).
\end{itemize}
\end{dfn}
Using the unique predicate $I$, we can form atomic formulas.
\begin{dfn}[atom, literal, clause]\label{dfn:atom}
An atom in $\il$ is defined by: if $\ttt_i$ ($i=1,2,3$) are terms in
$\il$, then $I(\ttt_1,\ttt_2,\ttt_3)$ is an atom. A literal is
defined to be an atom (called positive literal) or its negation
(called negative literal). A clause is the disjunction of a finite
set of literals.
\end{dfn}
Formulas in $\il$ are defined in the standard way.
\begin{dfn}[formula]\label{dfn:i-formula}
A formula in $\il$ is an expression involving  atoms and connectives
$\neg,\wedge,\vee$, which can be formed using the rules:
\begin{itemize}
\item [(i)] Any  atom is a formula.
\item [(ii)] If $\mathcal{A}$ and $\mathcal{B}$ are  formulas, then
so are $(\neg\mathcal{A})$, $(\mathcal{A}\wedge\mathcal{B})$, and
$(\mathcal{A}\vee\mathcal{B})$.
\end{itemize}
\end{dfn}
In the rest of this paper, we shall sometimes omit parentheses, as
long as no ambiguity is introduced.

Since implication statement are convenient for expressing inference
rules, we define $(\mathcal{A}\rightarrow\mathcal{B})$ as an
abbreviation of $((\neg\mathcal{A})\vee\mathcal{B})$. As a
consequence, each  clause
\begin{equation}\label{eq:clause}
\bigvee_{i=1}^k \neg I(\ttt_{1i},\ttt_{2i},\ttt_{3i})\vee
\bigvee_{j=1}^l I(\ttt_{1,k+j},\ttt_{2,k+j},\ttt_{3,k+j})
\end{equation}
can be equivalently represented as an implication (or rule)
\begin{equation}\label{eq:disjunct}
\bigwedge_{i=1}^k I(\ttt_{1i},\ttt_{2i},\ttt_{3i})\rightarrow
\bigvee_{j=1}^l I(\ttt_{1,k+j},\ttt_{2,k+j},\ttt_{3,k+j}).
\end{equation}
Clauses are of particular importance in axiomatization of
independency models.
\begin{dfn}[Horn and disjunctive clauses]\label{dfn:horn-clause}
For a clause $\mathcal{C}$ of form Eq.~\ref{eq:disjunct},
$\mathcal{C}$ is called a Horn clause if $l\leq 1$, and called
disjunctive otherwise.
\end{dfn}
Above we introduced the syntactic part of $\il$. Next we turn to
semantic notions.
\begin{dfn}[valuation]\label{dfn:valuation}
Let $M$ be an independency model defined on $U$. A valuation in $M$
is a function $v:\{\tx_1,\tx_2,\cdots\}\rightarrow 2^U$.
\end{dfn}
Valuations can be extended in a natural way to terms in $\il$.
\begin{dfn}[valid valuation]\label{dfn:int}
Let $\mathcal{A}$ be a formula, and let $M$ be an independency model
defined on $U$. A valuation $v$ in $M$ is \emph{valid} for
$\mathcal{A}$ if for each atom $I(\ttt_1,\ttt_2,\ttt_3)$ appeared in
$\mathcal{A}$, $v({\ttt}_1),v({\ttt}_2)$, and $v({\ttt}_3)$ are
pairwise disjoint, where $v(\ttt)$ is the valuation of $\ttt$ in
$M$.
\end{dfn}
The notion of satisfaction is defined in the standard way. Note that
if $v$ is valid for $\mathcal{A}$ in $M$, then it is also valid for
any sub-formula $\mathcal{B}$ of $\mathcal{A}$ in $M$. The following
definition is therefore well-defined.
\begin{dfn}[satisfaction]\label{dfn:sat}
Let $\mathcal{A}$ be a formula, and let $M$ be an independency model
defined on $U$. A valuation $v$ in $M$ is said to satisfy
$\mathcal{A}$ if $v$ is valid for $\mathcal{A}$ and it can be shown
inductively to do so under the following conditions.
\begin{itemize}

\item $v$ satisfies  atom $I(\ttt_1,\ttt_2,\ttt_3)$ if
$(v({\ttt}_1),v({\ttt}_2),v({\ttt}_3))\in M$.

\item $v$ satisfies $\neg\mathcal{B}$ if $v$ does not satisfies $\mathcal{B}$.

\item $v$ satisfies $\mathcal{B}\vee\mathcal{C}$ if either
$v$ satisfies $\mathcal{B}$ or $v$ satisfies $\mathcal{C}$.

\item $v$ satisfies $\mathcal{B}\wedge\mathcal{C}$ if
$v$ satisfies both $\mathcal{B}$ and $\mathcal{C}$.
\end{itemize}
We say $M$ satisfies $\mathcal{A}$, in notation
$M\models\mathcal{A}$, if all valid valuations of $\mathcal{A}$ in
$M$ satisfy $\mathcal{A}$.
\end{dfn}

The following proposition is a consequence of the definition of
$\rightarrow$.
\begin{prop}\label{prop:sat}
Let $\mathcal{A},\mathcal{B}$ be two formulas, and let $M$ be an
independency model defined on $U$. Then
$M\models\mathcal{A}\rightarrow\mathcal{B}$ iff for any valid
valuation $v$ of $\mathcal{A}\rightarrow\mathcal{B}$ in $M$, $v$
satisfies $\mathcal{A}$ implies $v$ satisfies $\mathcal{B}$.
\end{prop}
For a clause we have the following characterization.
\begin{coro}\label{coro:rule}
Let $\mathcal{C}$ be a clause of form Eq.~\ref{eq:disjunct}, and let
$M$ be an independency model defined on $U$. Then
$M\models\mathcal{C}$ iff the following condition holds:
\begin{itemize}
\item for any valid valuation $v$ of $\mathcal{C}$ in $M$, if
$(v(\ttt_{1i}),v(\ttt_{2i}),v(\ttt_{3i}))\in M$ for \underline{all}
$1\leq i\leq k$, then $(v(\ttt_{1j}),v(\ttt_{2j}),v(\ttt_{3j}))\in
M$ for \underline{some} $k+1\leq j\leq k+l$.
\end{itemize}
\end{coro}

Given a family of independency models $\mathbb{M}$ and a (finite or
countable) set of formulas $\mathbb{F}$ in $\il$, we now formalize
the notion that $\mathbb{B}$ can be axiomatically characterized by
$\mathbb{F}$.
\begin{dfn}[axiomatization]\label{dfn:axiom}
A family of independency models $\mathbb{M}$ can be completely
characterized by a set of formulas $\mathbb{F}$ in $\il$ if the
following condition holds for any independency model $M$:
\begin{equation}\label{eq:axiom}
M\in\mathbb{M}\Leftrightarrow(\forall\mathcal{B}\in\mathbb{F})M\models\mathcal{B}.
\end{equation}
We say $\mathbb{M}$ has a finite (countable, resp.) axiomatization
if it can be completely characterized by a finite (countable, resp.)
set of \emph{formulas} in $\il$.
\end{dfn}
Since each formula in $\il$ is semantically equivalent to the
conjunction of a set of finite clauses, we need only consider
clauses.
\begin{prop}\label{prop:axiom}
A family of independency models $\mathbb{M}$ has a finite
(countable, resp.) axiomatization iff it can be completely
characterized by a finite (countable, resp.) set of \emph{clauses}
in $\il$.
\end{prop}

Analogous to propositional calculus, we have the completeness
theorem for $\il$.

\begin{thm}\label{thm:axiom}
Suppose $\mathbb{M}$ is axiomatically characterized by $\mathbb{F}$.
Let $\Sigma$ be a set of formulas, $\mathcal{A}$ be a formula. Then
the following two conditions are equivalent.
\begin{itemize}
\item [(1)] $\Sigma\models_{\mathbb{M}}\mathcal{A}$: for any model $M$ in
$\mathbb{M}$, if $M$ satisfies all formulas in $\Sigma$, it also
satisfies $\mathcal{A}$;

\item [(2)] $\Sigma\vdash_{\mathbb{F}}\mathcal{A}$: $\mathcal{A}$ is
deducible from\footnote{In the sense of logic deduction.} $\Sigma$
by using axioms in $\mathbb{F}$.
\end{itemize}
\end{thm}
In particular, we have
\begin{coro}\label{coro:axiom2}
Let $\mathbb{M}$ and $\mathbb{F}$ be as in the above theorem. For a
set $\Gamma$ of independence statements
$\{I(\ttt_{i1},\ttt_{i2},\ttt_{i3}):1\leq i\leq k\}$ and an
independence statement $\gamma=
I(\ttt_{k+1,1},\ttt_{k+1,2},\ttt_{k+1,3})$,  we have
$\Gamma\models_{\mathbb{M}}\gamma$ iff $\gamma$ is deducible from
$\Gamma$ by using axioms in $\mathbb{F}$.
\end{coro}

\section{Sub-models}
In this section, we consider sub-independency models.

\begin{dfn}[sub-model]\label{dfn:submodel}
Let $M$ be an independency model defined on $U$, and let $V$ be a
subset of $U$. We call $M|_V=\{(A,C,B)\in M: A,B,C\subseteq V\}$ the
sub-independency model (or simply sub-model) of $M$ on $V$.
\end{dfn}
The following result asserts that if an independency model satisfies
a formula, so does its sub-model.
\begin{prop}\label{prop:hier}
Let $\mathcal{A}$ be a formula, and let $M$ be an independency model
defined on $U$. For any subset $V$ of $U$, if $M$ satisfies
$\mathcal{A}$, then so does $M|_V$.
\end{prop}
\begin{proof} This is because any valuation $v$ in $M|_V$ is also
a valuation in $M$.
\end{proof}

An interesting question arises naturally. Given an independency
model $M$ on $U$, suppose $M$ is a CI relation (or graph-isomorph,
or causal model), and $V\subseteq U$. Is sub-model $M|_V$ also a CI
relation (or graph-isomorph, or causal model)? This is important for
a family of independency models $\mathbb{M}$ to be axiomatizable.
Actually, if $\mathbb{M}$ is not closed under sub-models, then it
cannot be axiomatically characterized by any set of formulas.

Given a joint probability $P(\cdot)$, write $M$ for the CI relation
on $U$ induced by $P(\cdot)$, i.e. for any pairwise disjoint subsets
$A, B, C$ of $U$ the tuple $(A,C,B)$ is an instance of $M$ if and
only if $A$ and $B$ are conditionally independent given $C$ (see
Def. \ref{dfn:prob_model} and Def. \ref{dfn:independency_relation}).
We claim that, for a nonempty subset $V$ of $U$, $M|_V$, the
restriction of $M$ on $V$, is a CI relation on $V$. This is because
$M|_V$ is induced by the joint probability $P|_V(\cdot)$, which is
obtained from $P(\cdot)$ by computing the marginal probability of
$P$ on $V$.

A similar conclusion holds for graph-isomorphs.

\begin{lemma}\label{lemma:graph-iso}
Let $G=(U,E)$ be an undirected graph on $U$, and let $V$ be a
nonempty proper subset of $U$. Define an undirected graph
$G^\prime=(V,E^\prime)$ as follows: for any two nodes
$\alpha,\beta\in V$, $(\alpha,\beta)\in E^\prime$ iff there is a
path $p$ from $\alpha$ to $\beta$ in $G$ such that all other nodes
in $p$ are contained in $U-V$. Then
\begin{equation}
\langle\alpha|C|\beta\rangle_G \Leftrightarrow
\langle\alpha|C|\beta\rangle_{G^\prime}
\end{equation}
 for any $\alpha,\beta\in V$, and any $C\subset V$.
\end{lemma}

\begin{proof}
Suppose $\langle\alpha|C|\beta\rangle_{G^\prime}$. We show $C$
separates $\alpha$ from $\beta$ in $G$.  For each path
\[p=\alpha\gamma_1\gamma_2\cdots\gamma_m\beta\ \ (m\geq 0)\] in $G$,
we show $m\geq 1$ and some $\gamma_i$ is contained in $C$. Since
$\langle\alpha|C|\beta\rangle_{G^\prime}$, $(\alpha,\beta)$ is not
an edge in $G^\prime$. By definition, we know (i) $(\alpha,\beta)$
is not an edge in $G$, hence $m\geq 1$;  and (ii) some node
$\gamma_i$ must be contained in $V$. Suppose
$\gamma_{i_1},\gamma_{i_2},\cdots,\gamma_{i_k}$ ($1\leq
i_1<i_2<\cdots<i_k\leq m$) are all those nodes in $V$. Since nodes
between $\gamma_{i_u}$ and $\gamma_{i_{u+1}}$ (and those between
$\alpha$ and $\gamma_{i_1}$, and between $\gamma_{i_k}$ and $\beta$)
must be contained in $U-V$. By definition of $G^\prime$, we know
$(\alpha,\gamma_{i_1}), (\gamma_{i_u},\gamma_{i_{u+1}}),
(\gamma_{i_k},\beta)$ are all edges in $G^\prime$. Therefore
\[p^\prime=\alpha\gamma_{i_1}\gamma_{i_2}\cdots\gamma_{i_k}\beta\]
is a path in $G^\prime$. By
$\langle\alpha|C|\beta\rangle_{G^\prime}$, we know some
$\gamma_{i_u}$ must be in $C$. This shows that $C$ separates
$\alpha$ from $\beta$ for every path $p$ in $G$.

On the other hand, suppose $\langle\alpha|C|\beta\rangle_G$. We show
$C$ separates $\alpha$ from $\beta$ in $G^\prime$. For each path
\[p=\alpha\gamma_1\gamma_2\cdots\gamma_m\beta\ \ (m\geq 0)\] in
$G^\prime$, we show some $\gamma_i$ is contained in $C$.

Write $\gamma_0$ and $\gamma_{m+1}$ for $\alpha$ and $\beta$. Note
that if $(\gamma_i,\gamma_{i+1})$ is not an edge in $G$, then by
definition there is a path $p_i$ in $G$ from $\gamma_i$ to
$\gamma_{i+1}$ such that all other nodes in $p_i$ are contained in
$C$. Concatenating paths $p_0,p_1,\cdots,p_m$ we obtain a new `path'
in $G$ from $\alpha$ to $\beta$ which satisfies the following
condition:
\begin{equation*}\label{eq:condition}
\mbox{Each\ node\ is\ either\ in\ } p\ \mbox{or\ in\ } U-V.
\end{equation*}
In this `path' identical nodes may occur several times. With proper
modifications, we obtain a shortened path $p^\prime$ in $G$ which
also satisfies condition (\ref{eq:condition}). By our assumption
that $\langle\alpha|C|\beta\rangle_G$, we know some node in
$p^\prime$ must be contained in $C$. But by $C\subseteq V$, this
shows some $\gamma_i$ must be in $C$. Hence $C$ separates $\alpha$
from $\beta$ for every path $p$ in $G^\prime$.
\end{proof}

\begin{prop}\label{prop:graph-iso}
Let $M$ be a graph-isomorph on $U$. For a nonempty subset $V$ of
$U$, $M|_V$ is also a graph-isomorph.
\end{prop}
\begin{proof}
Suppose  $M$ is represented by an undirected graph $G=(U,E)$. We
show $M|_V$ can be represented by the undirected graph
$G^\prime=(V,E^\prime)$ constructed in Lemma~\ref{lemma:graph-iso},
i.e. for  for any pairwise disjoint subsets $A,B,C$ of $V$, we have
$I(A,C,B)_{M|_V}$ iff $\langle A|C|B\rangle_{G^\prime}$. By
definition of graph separation, for a graph $G^\ast$ we know
$\langle A|C|B\rangle_{G^\ast}$ iff $(\forall\alpha\in
A)(\forall\beta\in B)\langle\alpha|C|\beta\rangle_{G^\ast}$. By
Lemma~\ref{lemma:graph-iso}, for any $\alpha,\beta\in V$ and any
$C\subset V$ we have $\langle\alpha|C|\beta\rangle_G \Leftrightarrow
\langle\alpha|C|\beta\rangle_{G^\prime}$. Therefore $\langle
A|C|B\rangle_{G^\prime}$ iff $\langle A|C|B\rangle_{G}$ for any
pairwise disjoint subsets $A,B,C$ of $V$. Since $M$ is representable
by $G$, it is clear that $M|_V$ is also representable by $G^\prime$.
\end{proof}

But the following example shows that this is not true for causal
models.
\begin{figure}\centering
\includegraphics[width=.5\textwidth]{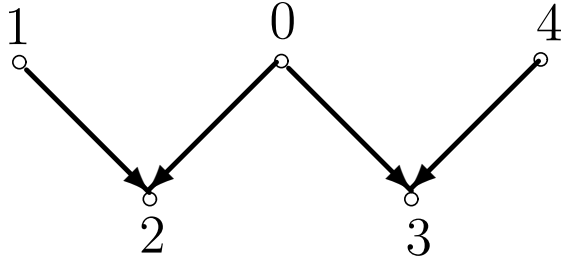}
\caption{A DAG $D$ on $U=\{0,1,2,3,4\}$.}\label{fig:c-ex}
\end{figure}

\begin{ex}\label{ex:c-ex}
Let $M$ be the causal model representable by the DAG $D$ given in
Fig.~\ref{fig:c-ex}, and let $V=\{1,2,3,4\}$. The sub-independency
model $M|_V$ is not representable by any DAG.
\end{ex}
To prove this conclusion, we use the notation $D(\alpha,\beta)$ to
express the fact that in $M|_V$ there is no $C\subset V$ such that
$I(\alpha,C,\beta)$ is true. It is clear that the following
independency statements holds in $M|_V$:
\begin{itemize}
\item $D(1,2)$, $D(3,4)$, $D(2,3)$;
\item $I(1,\varnothing, 3)$, $I(1,4,3)$, $I(2,\varnothing,4)$,
$I(2,1,4)$.
\end{itemize}
In a DAG $D=(U,\overrightarrow{E})$, for any two nodes
$\alpha,\beta\in U$, it is well known that
$(\alpha,\beta)\in\overrightarrow{E}$ or
$(\beta,\alpha)\in\overrightarrow{E}$ iff no $C\subseteq U$ can
$d$-separates $\alpha$ from $\beta$.

Suppose $M|_V$ is representable by some DAG $D^\prime$ defined on
$V$. By $D(1,2)$, $D(3,4)$, and $D(2,3)$ we know in $D^\prime$ node
1 is connected to node 2, node 2 is connected to node 3, and node 3
is connected to node 4. This shows that $p=1234$ is a path from node
$1$ to node $4$. But by $I(1,\varnothing, 3)_{M|_V}$ and
$p^\prime=123$ is a path from node 1 to node 3, we know in
$D^\prime$ we should have $1\rightarrow 2\leftarrow 3$. Similarly,
for nodes 2 and 4, we should also have $2\rightarrow 3\leftarrow 4$
in $D^\prime$. This is impossible since $2\rightarrow 3$ and
$2\leftarrow 3$ cannot appear together in the same DAG.

This proves that $M|_V$ has no DAG representation, hence is not
causal.

As a corollary of this example and Prop.~\ref{prop:hier} we have
\begin{thm}\label{thm:main}
Causal models have no complete axiomatic characterization.
\end{thm}
\begin{proof}
Let $\Gamma=\{\mathcal{A}_1,\mathcal{A}_2,\cdots\}$ be the set of
clauses that are satisfied by all causal models. In particular, the
causal model $M$ given in Example~\ref{ex:c-ex} satisfies each
$\mathcal{A}_i$. By Prop.~\ref{prop:hier} we know $M|_V$ also
satisfies each $\mathcal{A}_i$. Since $M|_V$ is not causal, the
infinite set $\Gamma$ (let alone finite subsets of $\Gamma$) cannot
provide a complete characterization for causal models.
\end{proof}

\section{Discussion}
We have shown that it is impossible to give a complete axiomatic
characterization for causal models. This is different from the
results obtained in \cite{PearlP85} and \cite{Studeny92}. In
\cite{PearlP85}, Pearl and Paz proved that graph-isomorphs have a
complete characterization by five axioms (4 Horn, 1 disjunctive).
Since a sub-model of a graph-isomorph also satisfies these axioms,
it is clear that sub-models of graph-isomorphs are graph-isomorphs.
We gave a method for constructing such a graph representation.

\studeny\ \cite{Studeny92} showed that there is no finite
axiomatization for CI relations by using Horn clauses. More
positively, he also showed that there exist an infinite set of Horn
clauses that completely characterize CI relations. But it is still
unknown whether CI relations have finite axiomatization by using
arbitrary clauses (Horn or disjunctive).

The class of sub-models of causal models seems useful when (unknown)
hidden variables are involved. As for axiomatization, a result by
Geiger (see \cite[Exercises 3.7]{Pearl88}) suggests that it may have
no finite characterization by Horn axioms.



\bibliographystyle{plain}   
\bibliography{causal}
\end{document}